\documentclass{article}

\usepackage{microtype}
\usepackage{graphicx}
\usepackage{subfigure}
\usepackage{booktabs} 
\usepackage{amsmath,amsthm,amssymb}
\usepackage{xfrac}
\usepackage{hyperref}
\usepackage{mathtools}

\DeclareMathOperator{\E}{\mathbb{E}}
\usepackage[accepted]{icml2021}
\newtheorem{corollary}{Corollary}

\newtheorem{theorem}{Theorem}

\icmltitlerunning{Functional Space Analysis of Local GAN Convergence}

\begin{document}

\twocolumn[
\icmltitle{Functional Space Analysis of Local GAN Convergence}



\icmlsetsymbol{equal}{*}

\begin{icmlauthorlist}
\icmlauthor{Valentin Khrulkov}{ya}
\icmlauthor{Artem Babenko}{ya,hse}
\icmlauthor{Ivan Oseledets}{sk}
\end{icmlauthorlist}

\icmlaffiliation{ya}{Yandex, Russia}
\icmlaffiliation{hse}{National Research University Higher School of Economics, Moscow, Russia}
\icmlaffiliation{sk}{Skolkovo Institute of Science and Technology, Moscow, Russia}
\icmlcorrespondingauthor{Valentin Khrulkov}{khrulkov.v@gmail.com}

\icmlkeywords{Machine Learning, ICML}

\vskip 0.3in
]



\printAffiliationsAndNotice{}  

\begin{abstract}
Recent work demonstrated the benefits of studying continuous-time dynamics governing the GAN training. However, this dynamics is analyzed in the model parameter space, which results in finite-dimensional dynamical systems. We propose a novel perspective where we study the local dynamics of adversarial training in the general functional space and show how it can be represented as a system of partial differential equations. Thus, the convergence properties can be inferred from the eigenvalues of the resulting differential operator. We show that these eigenvalues can be efficiently estimated from the target dataset before training. Our perspective reveals several insights on the practical tricks commonly used to stabilize GANs, such as gradient penalty, data augmentation, and advanced integration schemes. As an immediate practical benefit, we demonstrate how one can a priori select an optimal data augmentation strategy for a particular generation task.

\end{abstract}
\section{Introduction}
Generative Adversarial Networks (GANs) \citep{goodfellow2014generative} allow for efficient learning of complicated probability distributions from samples. However, training such models is notoriously complicated: the dynamics can exhibit oscillatory behavior, and the convergence can be very slow. To alleviate this, a number of practical tricks, including the usage of data augmentation \citep{karras2020training,zhao2020image,zhang2019consistency} and optimization methods inspired by numerical integration schemes for ODEs \citep{qin2020training} have been proposed. Nevertheless, the theory explaining the success of these methods is only starting to catch up with practice.

The standard way of training a GAN is to use simultaneous gradient descent. In \citet{nagarajan2017gradient} it was shown that under mild assumptions, this method is locally convergent. The authors \citet{mescheder2018training} demonstrated that when these assumptions are not satisfied, the usage of regularization methods such as the gradient penalty \citep{roth2017stabilizing} or consensus optimization \citep{mescheder2017numerics} is required to achieve convergence. In \citet{balduzzi2018mechanics,nie2020towards,liang2019interaction} the local convergence is studied based on the eigenvalues of the Jacobian of the dynamics near the equilibrium. Such analysis involves the parameterization of the generator and the discriminator by neural networks and usually does not take into account the properties of the target distribution.

Our key idea is that this dynamics can be efficiently analyzed in the \emph{functional space}. This is achieved by constructing a local quadratic approximation of the GAN training objective and writing down the dynamics as a system of partial differential equations (PDEs). The differential operator underlying this system has a remarkably simple form, and its spectrum can be analyzed in terms of the fundamental properties of the target distribution. Furthermore, we show that the convergence of the dynamics is determined by the \emph{Poincar\'{e} constant} of this distribution. Intuitively, this constant describes the connectivity properties of a distribution; for instance, it is smaller for multimodal datasets with disconnected modes, where GAN convergence is known to be slower. This connection is practically important since the Poincar\'{e} constant of a dataset can be easily estimated a priori, and the larger it is, the better is the expected convergence of GAN. Thus, we can analyze common techniques that alter the train set in terms of their effect on the Poincar\'{e} constant. For instance, one can choose a proper augmentation strategy that increases the Poincar\'{e} constant the most.
The main contributions of our paper are:
\vspace{-3mm}
\begin{itemize}
    \item We develop a linearized GAN training model in the functional space in the form of a PDE.
    
    \item We derive explicit formulas connecting the eigenfunctions and eigenvalues of the resulting PDE operator with the target distribution properties. This connection provides a theoretical justification for common GAN stabilization techniques.
    
    \item We describe an efficient, practical recipe for choosing optimal parameters for gradient penalty and data augmentations for a particular dataset.
\end{itemize}

\section{Second-order approximation of the GAN objective}
We assume that data samples $\{X_{i}\}_{i=1}^N \in \mathbb{R}^d$ are produced by a target probability measure $\mu$. The generator function $G(z)$ is a deterministic mapping from a latent space $\mathbb{R}^{l}$ to $\mathbb{R}^d$; $z$ is sampled from a known probability measure $\mu_z$. The discriminator $D$ is a real-valued function on $\mathbb{R}^d$ which goal is to distinguish between real and fake samples. The GAN objective can be written as a min-max problem 
\begin{equation}\label{pdegan:saddle0}
    \max_G \min_D f(D, G),
\end{equation}
where
\begin{equation}\label{pdegan:mainobjective}
    f(D, G) = \E_{x \sim \mu} \phi_1(D(x)) +  \E_{z \sim \mu_z} \phi_2(D(G(z)),
\end{equation}
and $\phi_1$, $\phi_2$ are some scalar convex twice differentiable functions.
For example, for the LSGAN \cite{mao2017least}
$$\phi_1(x) = \frac{1}{2} (x - 1)^2, \quad \phi_2(x) = \frac{1}{2} (x + 1)^2.$$
We analyze the behaviour of a GAN in a neighborhood of the Nash equilibrium corresponding to an optimal discriminator function $D_*$ and an optimal generator function $G_*$. We make standard assumptions \citep{nagarajan2017gradient,mescheder2018training,nie2020towards} that the optimal discriminator satisfies $D_*(x)=0 \ \forall x$ and the measure produced by $G_*$ is equal to $\mu$. Similarly, we also assume that $\phi_1''(0) + \phi_2''(0) \geq 0$ and $\phi_1'(0)=-\phi_2'(0)\neq 0$. In this setting we can recover many common GAN variations including the vanilla GAN \citep{goodfellow2014generative}, WGAN \citep{arjovsky2017wasserstein}, LSGAN \citep{mao2017least}.

The standard way of solving the min-max problem is to represent $D$ and $G$ by some parametric models and solve it in the parameter space. The convergence of the resulting dynamics can be studied by linearizing it around the equilibrium point \citep{nagarajan2017gradient,mescheder2018training,nie2020towards}. This is equivalent to the construction of the quadratic approximation of $f(D, G)$ in the parameter space. In this paper, we do the quadratic approximation first, obtaining a new saddle point problem in a \emph{functional} space that approximates the local dynamics of standard GAN training methods.
This representation has a remarkably simple form, and its properties can be studied in detail. 
\paragraph{Technical assumptions.}
We assume that $\mu$ is a measure with positive density $\rho > 0$ on $\mathbb{R}^d$, i.e., $d\mu = \rho(x)dx$. In practice, this can be (formally) achieved by smoothing a discrete data distribution with Gaussian noise. As the functional space, we consider $L_2(\mu)=L_2(\mu, \mathbb{R}^d)$ equipped with a natural weighted inner product:
\begin{equation*}
    \langle u_1, u_2 \rangle_{\mu} = \int u_1(x) u_2(x) d \mu.
\end{equation*}
We will also often utilize differential operators, which should be understood in the distributional sense. In these cases, the functions under study are assumed to be the elements of the corresponding Sobolev space. Specifically, we will utilize the weighted Sobolev spaces $H^1(\mu)$ and $H^2(\mu)$. These are spaces of generalized functions such that their distributional first order and second order derivatives respectively lie in the $L_2(\mu)$.

We start by finding the second-order approximation to the GAN objective near Nash equilibrium in the functional space. We assume that $G_*$ is invertible, i.e., for any $x \sim \mu$ there exists a unique $z\sim \mu_{z}$ such that $G_*(z) = x$.
Let \mbox{$(\delta D, \delta G)$} be the \emph{functional variations} of the optimal discriminator and generator respectively. 
\begin{theorem}{\label{pdegan:thmapprox}}
Let $D = D_* + \delta D = \delta D, G = G_* + \delta G$.
Let us denote $u(x)=\delta D(x), v(x) = (\delta G)(G^{-1}_*(x))$. 
Then, \eqref{pdegan:mainobjective} can be approximated as 
\begin{equation*}
    f(D, G) = f_0 + g(u, v) + \mbox{higher order terms in $u$, $v$},
\end{equation*}
where
\begin{equation*}
    g(u, v) = \alpha  \langle u, u \rangle_{\mu} + \beta \langle \nabla_x u, v\rangle_{\mu},
\end{equation*}
and $\alpha = \frac{1}{2} \left(\phi''_1(0) + \phi''_2(0)\right), \beta = \phi'_2(0)$ and~$f_0=f(D_*, G_*)$.
\end{theorem}
\begin{proof}
The proof is straightforward: we use Taylor expansion to approximate both terms in the objective up to second order and also use the fact that in the Nash equilibrium the first order terms sum up to $0$.
\begin{equation*}
\begin{split}
    &f(D_* + \delta D, G_* + \delta G) = f(\delta D, G + \delta G)  = \\
    &= \E_{x \sim \mu} \phi_1(\delta D(x)) + \E_{z \sim \mu_z} \phi_2(\delta D(G(z) + \delta G(z)) \\
    &= f_0 + \alpha \E_{x \sim \mu}  \delta D^2 + \beta \E_{z \sim \mu_z}  \langle \nabla_x \delta D(G(z)), \delta G(z) \rangle + \\
    & + \mbox{higher order terms},
    \end{split}
\end{equation*}
From the definitions of $\langle \cdot, \cdot \rangle_{\mu}$, $u$ and $v$ the statement of the theorem follows.
\end{proof}
\section{Continuous gradient descent-ascent in the functional space}
Now we need to solve the min-max problem of the form
\begin{equation}\label{pdegan:approxsaddle}
    \min_u \max_v g(u, v), \quad g(u, v) = \alpha  \langle u, u \rangle_{\mu} + \beta \langle \nabla_x u, v\rangle_{\mu}.
\end{equation}
We will refer to \eqref{pdegan:approxsaddle} as \emph{linearized GAN problem} (lGAN). Since it is a local quadratic approximation of the original min-max problem, local behavior of GAN training methods can be understood on this task. The linearized GAN problem depends solely on the measure $\mu$, and we will see what are the particular fundamental properties of this measure that determine the convergence of the gradient-based method for solving \eqref{pdegan:approxsaddle}. 

We will use for $u$ the weighted Sobolev space $H^2(\mu)$ and for $v$ the weighted Sobolev space $H^1(\mu)$. 
The continuous descent-ascent flow in the functional space can be written as:
\begin{equation}\label{pdegan:odepdegan}
    u_t = -\nabla_u g(u, v), v_t = \nabla_v g(u, v),
\end{equation}
where $\nabla_u g$ and $\nabla_v g$ are Fr\'{e}chet derivatives of the functional $g(\cdot, \cdot)$ with respect to $u$ and $v$. It is important to define what is the meaning of \eqref{pdegan:odepdegan}, i.e.,  it should be understood in the weak sense. Select a test function $\hat{u}$ and take the scalar product of the first equation with it. We get
\begin{equation}
    \langle u_t, \hat{u} \rangle_{\mu} = -\langle\nabla_u g(u, v), \hat{u}\rangle_{\mu}, \, \forall \hat{u} \in H^2(\mu)
\end{equation}
Similar equation holds for $v$. The evaluation of the right-hand side is done by integration by parts, which leads to the system of the form
\begin{equation}\label{pdegan:weakform}
    \begin{split}
    &\langle u_t, \hat{u}\rangle_{\mu} = -\alpha \langle u, \hat{u} \rangle_{\mu} - \beta \langle \nabla_x \hat{u}, v \rangle_{\mu}, \, \forall \hat{u} \in H^2(\mu).\\
    &\langle v_t, \hat{v} \rangle_{\mu} = \beta \langle \nabla_x u, \hat{v} \rangle_{\mu}, \, \forall \hat{v} \in H^1(\mu).
    \end{split}
\end{equation}
By making use of the density $\rho$, equations \eqref{pdegan:weakform} can be rewritten in the strong form as
\begin{equation}\label{pdegan:eq1}
\begin{split}
    & \rho u_t = -\alpha \rho u + \beta \nabla_x \cdot (\rho v), \\
    & \rho v_t = \rho \beta \nabla_x u.
\end{split}
\end{equation}

We are interested in the asymptotic behavior of $u$ and $v$ when $t$ goes to infinity. 
It is completely described by the eigenvalues of the operator on the right-hand side of \eqref{pdegan:weakform}.
We will say that  $(u_{\lambda}, v_{\lambda})$ is an eigenfunction with an eigenvalue $\lambda$ if it satisfies the following system of equations.
\begin{equation}\label{pdegan:eigenproblem}
    \begin{split}
    &\lambda \langle u_{\lambda}, \hat{u}\rangle_{\mu} = -\alpha \langle u_{\lambda}, \hat{u}\rangle_{\mu} - \beta \langle\nabla_x \hat{u}, v_{\lambda}\rangle_{\mu}, \, \forall \hat{u} \in H^2(\mu).\\
    &\lambda \langle v_{\lambda}, \hat{v}\rangle_{\mu} = \beta \langle \nabla_x u_{\lambda}, \hat{v}\rangle_{\mu}, \forall \hat{v} \in H^1(\mu). 
    \end{split}
\end{equation}

\section{Eigenfunctions and eigenvalues of the linearized operator}

If $(u_{\lambda}, v_{\lambda})$ form an eigenfunction, then the solution of the time-dependent problem \eqref{pdegan:weakform} with the initial conditions $(u(0)=u_{\lambda}, v(0) = v_{\lambda})$ can be written as
\begin{equation}\label{pdegan:timedep}
    u(t) = u_{\lambda} \exp(\lambda t), \quad v(t) = v_{\lambda} \exp(\lambda t).
\end{equation}
We will see that $\{ (u_{\lambda}, v_{\lambda}) \}_{\lambda}$ form a basis in our functional space and thus an arbitrary solution can be written as a series of terms \eqref{pdegan:timedep}.
Thus, the real parts of the spectrum of the operator in \eqref{pdegan:weakform} should be less than $0$ in order for the system to have asymptotic stability. We will demonstrate that our system does not have eigenvalues with a positive real part but has a naturally interpretable kernel.
\subsection{The kernel.}

We find that the kernel has the following form.
\begin{corollary}
Let $(u_0, v_0)$ be an eigenfunction with $\lambda = 0$.  Then,
\begin{equation}\label{pdegan:kernel1}
   u_0 = C, \quad \langle\nabla_x \hat{u}, v_0\rangle_{\mu} = 0, \, \forall \hat{u} \in H^2(\mu),
\end{equation}
or in the strong form:
\begin{equation}\label{pdegan:kernel2}
   u_0 = C, \nabla_x \cdot (\rho v_0) = 0.
\end{equation}
Here $C$ is a constant such that $C\alpha=0$. I.e., for $\alpha\neq0$ we get $C=0$, and $C\in\mathbb{R}$ otherwise. 
\end{corollary}
\begin{proof}
From \eqref{pdegan:eigenproblem} we observe that the element $(u_0, v_0)$ of the kernel satisfies the following equations $\forall \hat{u} \in H^2(\mu)$.
\begin{equation}
    u_0 = C, -\alpha \langle C, \hat{u} \rangle_{\mu} - \beta \langle \nabla_x \hat{u}, v_{0} \rangle_{\mu} = 0. 
\end{equation}
Let us choose $\hat{u}=1$. From the second equation it follows that $\alpha C = 0$ as desired.
\end{proof}
This kernel has a straightforward interpretation.  If we add a function $v_0$ that satisfies \eqref{pdegan:kernel1} to the $G_*$, the mapped measure will be the same up to second order terms. Indeed, consider \eqref{pdegan:kernel1} in the strong form. We obtain
 that $\nabla_x \cdot (\rho v_0) = 0$.
Recall that the function $x + v_0(x)  = x + \delta G(G_*^{-1}(x))$ maps a sample $x$ to a sample from the synthetic density produced by $G_* + \delta G$, i.e. $v_0(x)$ can be considered as a \emph{velocity} of each sample. If samples evolve with velocity $v_0(x)$, the differential equation for the density takes the form
\begin{equation*}
    \rho_t = \nabla_x \cdot (\rho v_0) = 0,
\end{equation*}
i.e., the density is invariant under such transformation. We will refer to the condition \eqref{pdegan:kernel1} as the \emph{divergence-free} condition on $v_0$. 

\subsection{Non-zero spectrum.}{\label{sec:nonzerozpec}}
\subsubsection{Weighted Laplace operator}
We now address non-zero eigenvalues of our operator. It will be convenient to utilize the \emph{weighted Laplace operator} $\Delta_{\mu}$ which is defined (in the weak form) as follows.
\begin{equation}\label{pdegan:deflap}
    \langle \Delta_{\mu} w, \hat{w} \rangle_{\mu} \coloneqq -\langle \nabla_{x} w, \nabla_{x} \hat{w} \rangle_{\mu}, \, \forall \hat{w} \in H^2(\mu).
\end{equation}
In the strong form $\Delta_{\mu}$ takes the following form:
\begin{equation}
    \Delta_{\mu} w = \frac{1}{\rho}\nabla_{x} \cdot (\rho \nabla_{x} w),
\end{equation}
which results in the standard Laplacian in the case of the standard Lebesgue measure on $\mathbb{R}^d$. This operator commonly appears in the study of the diffusion processes \citep{coifman2006diffusion} and weighted heat equations \citep{grigoryan2009heat}.
\paragraph{Spectrum of the weighted Laplacian.}
In what follows we will use eigenvalues and eigenfunctions of $\Delta_{\mu}$. Firstly, this a self-adjoint non-positive definite operator and under mild assumptions on $\mu$ \citep{cianchi2011discreteness,grigoryan2009heat}, e.g., when $\rho$ decreases sufficiently fast, it has a discrete spectrum. Let us study it in more detail. Consider the eigenvalue problem for $-\Delta_{\mu}$ written in the weak form. 
\begin{equation}
    \langle\nabla_x \hat{w}, \nabla_x w_{\xi}\rangle_{\mu} = \xi \langle \hat{w}, w_{\xi} \rangle_{\mu} \quad \forall \hat{w} \in H^2(\mu).
\end{equation}
We note that $w_0 \equiv 1$ is the eigenfunction with $\xi=0$; thus, due to self-adjointness of $\Delta_{\mu}$ for every other eigenfunction $w_{\xi}$ we have $\langle 1, w_{\xi} \rangle_{\mu}=0$, i.e., it has zero mean with respect to $\mu$. 
\paragraph{The Poincar\'e constant of $\mu$.}
Let us consider the smallest \emph{non-zero} eigenvalue $\xi_{\min}$ of $-\Delta_{\mu}$. We obtain that the following inequality holds.
\begin{equation}{\label{pdegan:poincareineq}}
     \langle\nabla_x w, \nabla_x w \rangle_{\mu} \geq \xi_{\min} \langle w,  w \rangle_{\mu}, \langle 1, w\rangle_{\mu} = 0, \,  \forall w \in H^2(\mu), 
\end{equation}
and the exact minimizer of this inequality is achieved by $w_{\xi_{\min}}$ \citep{grigoryan2009heat}. The value of $\xi_{\min}$ (sometimes its inverse) is called the \emph{Poincar\'e constant} of the measure $\mu$ and often appears in Sobolev-type inequalities \citep{adams2003sobolev}. The exact values of $\xi_{\min}$ for a given measure are almost never known analytically. To make use of our results in practical settings, we propose a simple neural network based approach for estimation of $\xi_{\min}$. We discuss it in Section~\ref{sec:practical}.
\paragraph{What is it exactly?} We can provide an intuitive meaning of $\xi_{\min}$ by drawing an analogy with graph Laplacians. In this case, the second smallest eigenvalue of the graph Laplacian, called the Fiedler value \citep{fiedler1973algebraic}, reflects how well connected the overall graph is. Thus, the Poincar\'{e} constant is in a way a continuous analog of this constant, reflecting the ``connectivity'' of a measure. This is the property that we would expect to impact the GAN convergence, as based on rich empirical evidence, the datasets that are more ``disconnected'' (such as ImageNet) are very challenging to model. We provide experimental evidence in support of this intuitive explanation in Section~\ref{pdegan:secexp}.
\subsubsection{Spectrum of lGAN}
We are now ready to fully describe the spectrum of our lGAN model.
Let $\{ \xi_{i} \}_{i=1}^{\infty}$ be the non-zero spectrum of $-\Delta_{\mu}$ and $\{ w_{\xi_i} \}_{i=1}^{\infty}$ be the set of the corresponding eigenfunctions. Recall from Theorem \ref{pdegan:thmapprox} that the constants $\alpha$ and $\beta$ are defined purely in terms of functions $\phi_1$ and $\phi_2$. We now state our main result.
\begin{theorem}\label{pdegan:eigenfunthm}
The non-zero spectrum of \eqref{pdegan:eigenproblem} is described as follows.
\begin{itemize}

 \item The eigenvalues are given by $\{ \lambda^{\pm}_{i} \}_{i=1}^{\infty}$ where $\lambda^{\pm}_{i}$ are roots of the quadratic equation:
\begin{equation}{\label{pdegan:quadeig}}
\lambda^2 + \alpha \lambda + \beta^2 \xi_i = 0.
\end{equation}
\item The corresponding eigenfunctions are written in terms eigenfunctions of $-\Delta_{\mu}$ as follows.
\begin{equation}
(u_{\lambda_i^{\pm}},v_{\lambda_i^{\pm}}) = (w_{\xi_i}, \frac{\beta}{\lambda_i^{\pm}} \nabla_x w_{\xi_i}).    
\end{equation}
\end{itemize}
\end{theorem}
\begin{proof}
By putting $\hat{v} = \nabla_x \hat{u}$ into the second equation of \eqref{pdegan:eigenproblem}, we get
\begin{equation*}
\lambda (u_{\lambda}, \hat{u})_{\mu} = -\alpha (u_{\lambda}, \hat{u})_{\mu} - \frac{\beta^2}{\lambda} (\nabla_x \hat{u}, \nabla_x u_{\lambda})_{\mu},
\end{equation*}
which can be rewritten as
\begin{equation*}
(\nabla_x \hat{u}, \nabla_x u_{\lambda})_{\mu} = \frac{1}{\beta^2}\left( (-\alpha \lambda - \lambda^2) (u_{\lambda}, \hat{u})_{\mu} \right) = \xi(u_{\lambda}, \hat{u})_{\mu},
\end{equation*}
which means that $\xi$ is an eigenvalue of $-\Delta_\mu$, and $u_{\lambda}$ is its eigenfunction.
The eigenvalue $\lambda$ can be found from the solution of the quadratic equation \eqref{pdegan:quadeig}
\begin{equation}\label{pdegan:quadeqsolve}
     \lambda = \frac{-\alpha \pm \sqrt{\alpha^2 - 4 \beta^2 \xi}}{2}.
\end{equation}
\end{proof}
With the explicit formulas for eigenvalues and eigenvectors at hand, we are now ready to analyze the convergence of the problem \eqref{pdegan:weakform}.


\section{Convergence}
The following theorem expresses the solution in terms of the eigenfunctions and also provides the convergence estimates. 
\begin{theorem}\label{pdegan:expansion}
Let $u_0 \in H^2(\mu),  v_0 \in H^1(\mu)$ and \mbox{$\int u_0 d \mu = c_0$}. Then, these functions can be written as
\begin{equation*}
\begin{split}
    &u_0 = c_0 + \sum_{k=1}^{\infty} (c_k^{+} + c_k^{-})w_{\xi_k},\\ 
    &v_0 = \widetilde{v}_0 + \nabla_x V_0, \, V_0 = \sum_{k=1}^{\infty} \Big( c_k^{+} \frac{\beta}{\lambda_i^+} + c_k^{-} \frac{\beta}{\lambda_i^{-}}\Big) w_{\xi_k},
\end{split}
\end{equation*}
and $\widetilde{v}_0$ is divergence-free, i.e. 
 $\langle\nabla_x \hat{u}, \widetilde{v}_0\rangle_{\mu} = 0$.
The coefficients $c_k^+$ and $c_k^-$ can be obtained as the solution of the linear systems:
\begin{equation}
    \begin{pmatrix}
    1 & 1 \\
    \frac{\beta}{\lambda_i^+} & \frac{\beta}{\lambda_i^-} 
    \end{pmatrix}
    \begin{pmatrix}
    c_k^+ \\
    c_k^-
    \end{pmatrix} = 
    \begin{pmatrix}
    \langle u_0, w_{\xi_k} \rangle_{\mu} \\
    \langle V_0, w_{\xi_k} \rangle_{\mu}
    \end{pmatrix}    
\end{equation}

 With this expansion, the solution to \eqref{pdegan:weakform} is 
 \begin{equation}
 \begin{split}
     & u(t) = c_0 e^{-\alpha t} + \sum_{k=1}^{\infty} w_{\xi_k} \left(c_k^+ e^{\lambda^{+}_k t} + c_k^- e^{\lambda^{-}_k t} \right),  \\
     & v(t) = \widetilde{v}_0 + \nabla_x V(t), \\
     & V(t) = \sum_{k=1}^{\infty} w_{\xi_k} \left(c_k^+\frac{\beta}{\lambda_i^+} e^{\lambda^{+}_k t} + c_k^- \frac{\beta}{\lambda_i^-}e^{\lambda^{-}_k t} \right).
     \end{split}
 \end{equation}
 
For $\alpha > 0$ the norms of $u(t)$ and $V(t)$ can be estimated as
 \begin{equation*}
     \Vert u(t) \Vert_{\mu} \leq  2 \Vert u_0 \Vert_{\mu} e^{\eta t}, \Vert V(t) \Vert_{\mu} \leq C \Vert V_0 \Vert_{\mu} e^{\eta t},
 \end{equation*}
 where $\eta = \mathrm{Re}\Big(\frac{-\alpha + \sqrt{\alpha^2 - 4 \beta^2\xi_{\min}}}{2}\Big) < 0 $ is the maximal real part of the eigenvalues. 
\end{theorem}
\begin{proof}
The decomposition of $v_0$ into a potential part and divergence-free part is a direct generalization of the classical result for the ordinary divergence and gradient, known as the Helmholtz decomposition \citep{griffiths2005introduction}. The divergence-free part $\widetilde{v}_0$ belongs to the kernel of the operator, thus it stays constant. The dynamics of $u$ and $v$ follows from the completeness of the eigenbasis of $\Delta_{\mu}$ and the assumption that its spectrum is discrete, thus we can expand them in this basis.  From Theorem~\ref{pdegan:eigenfunthm} each component in the sum is an eigenfunction, thus its time dynamics is just $e^{\lambda^{\pm}_k t}$. For the constant term in $u(t)$, by substituting $\hat{u}=1$ in \eqref{pdegan:weakform}, we obtain the following ODE $\langle u_t, 1 \rangle_{\mu} = -\alpha \langle u, 1 \rangle_{\mu}$, from which the statement follows.
\end{proof}
\paragraph{Discussion.}
Theorem~\ref{pdegan:expansion} provides the exact representation of the solution in terms of the eigenfunctions of $\Delta_{\mu}$. The eigenvalues are given as the solution of the quadratic equation \eqref{pdegan:quadeig}, thus the spectral properties of $-\Delta_{\mu}$ completely determine the dynamics of the convergence. Specifically, we observe that the speed of convergence is determined by the value of $\xi_{\min}$, i.e., the lowest non-zero mode of the weighted Laplacian.

There are two distinct cases. If the first eigenvalue of $-\Delta_{\mu}$ satisfies 
 $\alpha^2 - 4\beta^2 \xi_{\min} \leq 0,$     
(i.e., $\xi_{\min}$ is large enough), then all the eigenvalues $\lambda_i^{\pm}$ are complex, and their real part is equal to $-\frac{\alpha}{2}$. Ideally, we would have $\alpha^2 - 4\beta^2 \xi_{\min} = 0$, which would provide is with the optimal convergence speed and no oscillations for the highest mode. 
Consider now the case of a small $\xi_{\min} \ll \frac{\alpha^2}{4\beta^2}$. In this case, $\eta$ will be close to $0$, which would result in a slow convergence rate. These observations can be used to explain the success of various practical GAN training methods, as we show in Section~\ref{pdegan:secexp}.

In the case is $\alpha=0$, which holds, for instance, for WGAN, we obtain the well-known purely oscillatory behavior \citep{nagarajan2017gradient}. We also note that for $\alpha > 0$, the average value of the discriminator exponentially decays to $0$. This resembles the convergence plots of state-of-the-art GANs, see, e.g., \citet[Figure 6b]{karras2020training}. 


\paragraph{Example: normal distribution.}
Consider a model example of the normal distribution, $\mu \sim \mathcal N (0, 1)$. Then, $\mu$ has the density $\rho(x) = \frac{1}{\sqrt{2{\pi}}} e^{-x^2/2}.$ The eigenfunctions and eigenvalues of $-\Delta_{\mu}$ can be computed explicitly. The strong form of the eigenproblem is 
\begin{equation*}
    \frac{d}{d x} \rho \frac{d w_{\xi}}{d x} = -\xi \rho w_{\xi},
\end{equation*}
i.e. $w_{\xi}$ satisfies
\begin{equation}\label{pdegan:hermite-eig}
  \frac{d^2 w_{\xi}}{dx^2}  - x \frac{d w_{\xi}}{dx} = -\xi w_{\xi}.
\end{equation}
The solution of \eqref{pdegan:hermite-eig} exists for $\xi_k \in \mathbb{Z}_{\geq 0}$ and the corresponding eigenfunction is the Hermite polynomial:
$$w_{\xi_k} = H_k(x) = (-1)^n e^{\frac{x^2}{2}} \frac{d^n }{d x^n} e^{-\frac{x^2}{2}}.$$
The smallest non-zero eigenvalue is $1$. Therefore, for the LSGAN model, we will have the discriminant in \eqref{pdegan:quadeig} always non-positive, and the convergence $u$ and $v$ will be exponential with the rate $e^{-\frac{t}{2}}$. The solution also will oscillate due to the presence of complex eigenvalues.
\paragraph{What about neural networks?}
The discussion so far focuses on the functional spaces. In reality, these functions are approximated by neural networks, and the dynamics is written in the parameters $\theta_D$ and $\theta_G$ of these networks. Local convergence analysis of such dynamics is possible in these parameters \cite{mescheder2018training,nie2020towards}; however, such analysis involves eigenvalues of the Jacobian of the loss. It is not easy to connect these properties to the fundamental properties of the measure. The functional space analysis shows this connection. The approximation by neural networks (or by any other parametric representation) can be thought of as a \emph{spatial discretization} of PDEs. If the number of parameters increases, the eigenvalues of the discretized problem should approximate the eigenvalues of the infinite-dimensional problem. It is worth noting that different discretizations (for example, different neural network architectures) may lead to different properties. One can compare such discretizations by looking at the quality of approximation of the eigenvalues of the weighted Laplace operator (see Section~\ref{sec:practical} for the algorithmic details). Also, such techniques used in practice as Jacobian regularization \citep{karras2020analyzing} can be considered as methods for choosing a better-conditioned discretization. Finally, for future research, we would like to mention that the lGAN problem is similar to the saddle point problems \cite{benzi2005numerical} appearing in mathematical modeling of fluid problems. For example, for robust discretization, one has to choose discretization spaces of $u$ (`pressure'), and $v$ (`velocity') such that they satisfy the famous LBB (Ladyzhenskaya-Babu\v{s}ka-Brezzi) condition \cite{boffi2013mixed, ladyzhenskaya1969mathematical} in order to get good convergence properties. This task requires systematic study, and we leave it for future research.
\section{Effects of common practices on the asymptotic convergence}\label{sec:common}
\paragraph{Regularization.}
Several studies have shown that if we penalize the norm of the gradient of the GAN objective with respect to the discriminator, it improves the convergence \citep{gulrajani2017improved,mescheder2018training}. In our case, it results in an additional term $\frac{\gamma}{2} E_{\mu} \Vert \nabla_D f(D, G) \Vert^2$. After linearization we obtain a regularized loss $\hat{g}(u, v)$.
\begin{equation*}
    \max_v \min_u \hat{g}(u, v), \quad \hat{g}(u, v) = g(u, v) + \frac{\gamma}{2}  \E_{\mu}  \Vert \nabla_x u \Vert^2.
\end{equation*}
New eigenvalue problem has the form
\begin{equation*}
        \begin{split}
    &\lambda \langle u_{\lambda}, \hat{u}\rangle_{\mu} = -\alpha \langle u_{\lambda}, \hat{u}\rangle_{\mu} - \beta \langle\nabla_x \hat{u}, v_{\lambda}\rangle_{\mu} - \gamma \langle \nabla_x u, \nabla_x \hat{u} \rangle, \\
    &\lambda \langle v_{\lambda}, \hat{v}\rangle_{\mu} = \beta \langle \nabla_x u_{\lambda}, \hat{v}\rangle_{\mu}, \forall \hat{v} \in H^1(\mu), \, \hat{u} \in H^2(\mu).
    \end{split}
\end{equation*}
The kernel stays the same, and moreover, the $u_{\lambda}$ component is also the same, which can be seen again by using $\hat{v} = \nabla_x \hat{u}$ in the second equation. The only difference is the connection between eigenvalues of the weighted Laplace operator and eigenvalues of the linearized problem, which changes as follows:
\begin{equation*}
    \lambda^2 + (\alpha + \gamma \xi_i) \lambda + \beta^2 \xi_i = 0, 
\end{equation*}
yielding
\begin{equation}\label{pdegan:regularizedeig}
    \lambda_i^{\pm} = \frac{-(\alpha + \gamma \xi_i) \pm \sqrt{(\alpha + \gamma \xi_i)^2 - 4 \beta^2 \xi_i}}{2}.
\end{equation}
\paragraph{Optimal parameters.}
One of the most interesting conclusions of our analysis is that the convergence is determined by the connection between $\alpha, \beta, \gamma$ and $\xi$. If we fix $\phi_1, \phi_2$ in advance, we can not control $\alpha$ and $\beta$ and convergence for a particular dataset will be determined by $\xi_{\min}$. I.e., for one dataset the loss may work well, but for another it may fail. An alternative approach is to optimize over the hyperparameters of the loss and regularizer taking the Poincar\'{e} constant into account.
Expression \eqref{pdegan:regularizedeig} gives a direct way to do that. The eigenvalue with the maximal real part corresponds to $\xi = \xi_{\min}$ and is given by the formula \begin{equation*}
    \lambda_{\max} = \frac{-(\alpha + \gamma \xi_{\min}) + \sqrt{(\alpha + \gamma \xi_{\min})^2 - 4 \beta^2 \xi_{\min}}}{2}.
\end{equation*}
The maximum is obtained if the discriminant is equal to $0$, which gives
\begin{equation*}
    \alpha + \gamma \xi_{\min} = 2 |\beta| \sqrt{\xi_{\min}}.
\end{equation*}
Another desirable property is the absence of oscillations. This is only possible if the quadratic function under the square root is non-negative. Since it is equal to zero at $\xi = \xi_{\min}$, it is necessary and sufficient for it to have non-negative derivative:
\begin{equation*}
     2 \gamma ( \alpha + \gamma \xi_{\min}) \geq 4 \beta^2 \rightarrow \gamma \geq \frac{|\beta|}{\sqrt{\xi_{\min}}}.
\end{equation*}
Simple analysis gives the following conditions for the parameters $\alpha, \beta, \gamma$ such that we have optimal local convergence rate and all the eigenvalues of the linearized operator are real:
\begin{equation}\label{pdegan:optpar}
\begin{split}
    &\alpha + \gamma \xi_{\min} = 2 |\beta| \sqrt{\xi_{\min}}, \\
    &\frac{|\beta|}{\sqrt {\xi_{\min}}} \leq \gamma \leq \frac{2 |\beta|}{\sqrt {\xi_{\min}}}, \, 0 \leq \alpha \leq \frac{|\beta|}{\sqrt{\xi_{\min}}}. 
    \end{split}
\end{equation}

\paragraph{Discrete approximation of the time dynamics.}
The dynamics \eqref{pdegan:weakform} can be written in the operator form as 
\begin{equation}
 w_t = \mathcal{L} w, \, w(0) = w_0, \, w = \begin{bmatrix} u \\ v \end{bmatrix}.
\end{equation}
The usage of gradient descent optimization for this problem is equivalent to the forward Euler scheme: 
\begin{equation*}
  w_{k+1} = w_k + \tau \mathcal{L} w_k,
\end{equation*}
and the eigenvalues of the discretized operator are equal to 
\begin{equation}
    \lambda_i^{e} = 1 + \tau \lambda_i,
\end{equation}
where $\lambda$ are given by Theorem~\ref{pdegan:eigenfunthm}. Since for all $\xi_i$ such that $\alpha^2 - 4 \beta^2 \xi_i < 0$ the eigenvalue $\lambda_i$ is complex, the modulus of the corresponding $\lambda^e_i$ is 
\begin{equation*}
    \vert \lambda_i^e \vert^2 = \left(1 - \frac{\tau \alpha}{2}\right)^2 + \tau^2 \left(4 \beta^2 \xi_i - \alpha^2\right).
\end{equation*}
Since  $\Delta_\mu$ is an unbounded operator under mild assumptions on $\mu$ 
there exists an eigenvalue $\xi_i$ of it that makes $\vert \lambda_i^e \vert > 1$, i.e. the Euler scheme is \emph{absolutely unstable}. On the discrete level, when the functions are approximated by a neural network, we might be working in a subspace that does not contain eigenvalues that are too large, and the method might actually converge; but this is very problem specific. As has been noted by \citet{qin2020training} the usage of more advanced time integration schemes leads to competitive GAN training results even without such functional constraints as spectral normalization \citep{miyato2018spectral}. If the parameters are selected in the range \eqref{pdegan:optpar}, then there are no complex eigenvalues, and we can make Euler scheme convergent. This requires regularization. Another important research direction is the development of more suitable time discretization schemes. The system \eqref{pdegan:weakform} belongs to the class of \emph{hyperbolic problems}, and special time discretization schemes have to be used. One class of such methods are total variation diminishing (TVD) schemes \cite{gottlieb1998total}. Note, that one of the methods considered in \citet{qin2020training} is the Heun method \citep{suli2003introduction} which is a second-order TVD Runge-Kutta scheme, which confirms its empirical efficiency. Thus, our analysis provides theoretical foundation for the results of \citet{qin2020training}. 

\paragraph{Data augmentation.}
One of the most successful practical tricks significantly improving GAN convergence have been the usage of data augmentation \citep{karras2020training,zhang2019consistency,zhao2020image}. In our framework, the usage of data augmentations is reflected in the shift of $\xi_{\min}$. Based on the intuition discussed earlier, higher values of $\xi_{\min}$ correspond to more ``connected'' distributions and allow for faster convergence. This is exactly what is intuitively achieved by data augmentation: we fill the ``holes'' in our dataset with new samples and make it more connected. We describe our experiments confirming this idea in Section~\ref{pdegan:secexp}.

\paragraph{Practical estimation of $\xi_{\min}$.}{\label{sec:practical}}
For practical analysis of GAN convergence, we need to estimate the value of the Poincar\'e constant $\xi_{\min}$ for a given dataset. This can be performed in the standard supervised learning manner. Recall from the definition that $\xi_{\min}$ is the minimizer of the following optimization problem (commonly called the Rayleigh quotient).
\begin{equation}{\label{eq:rayleigh}}
    \frac{\E_{\mu} \| \nabla_x f(x) \|^2}{\mathrm{Var}_{\mu} f(x)} \to \min_{f \in H^2(\mu)}.
\end{equation}
In practice we can parameterize $f$ with a neural network, and perform optimization of \eqref{eq:rayleigh} in a stochastic manner by replacing expectations over $\mu$ with their empirical counterparts (over a batch of inputs). Due to the scale invariance of the Rayleigh quotient, we employ spectral normalization \citep{miyato2018spectral} as an additional regularization; this also falls in line with commonly used discriminator architectures. 
To summarize, we utilize a neural network $f(x; \theta)$ and for a dataset $X = \lbrace X_i \rbrace_{i=1}^N$ consider the batched version of the following loss.
\begin{equation}
   \mathcal{L}(\theta) = \frac{\frac{1}{N}\sum_{i=1}^{N}\|\nabla_x f(X_i; \theta)\|^2}{\frac{1}{N} \sum_{i=1}^N f(X_i)^2 - (\frac{1}{N} \sum_{i=1}^N f(X_i))^2},
\end{equation}
and, correspondingly, $\xi_{\min} \approx \min_{\theta}\mathcal{L}(\theta)$.
Minimization of this loss function can be performed with standard optimizers such as SGD or Adam \citep{kingma2014adam}.
\section{Experiments}{\label{pdegan:secexp}}
\paragraph{Experimental setup.} Our experiments are organized as follows. We start by numerically investigating the value of $\xi_{\min}$ and impact of formulas from Section~\ref{sec:common} on the convergence on synthetic datasets. We then study the more practical CIFAR-10 \citep{krizhevsky2009learning}
dataset. Firstly, we show the correlation between the $\xi_{\min}$ obtained for various augmented versions of the dataset and FID values obtained for the corresponding GAN. We then show that the similar correlation holds when we perform \emph{instance selection}, a recently proposed technique shown to improve GAN convergence \citep{devries2020instance}. For the synthetic datasets, our experiments were performed in JAX \citep{jax2018github} using the ODE-GAN code available at GitHub\footnote{\url{https://github.com/deepmind/deepmind-research/tree/master/ode_gan}}. For CIFAR-10 experiments we utilized PyTorch \citep{NEURIPS2019_9015}. Our experiments were performed on a single NVidia V-100 GPU.

\paragraph{Gaussian Mixture.}
In order to study the effect of $\xi_{\min}$ and the choice of $\alpha$ and $\gamma$ on the training procedure, we set up a simple one-dimensional test: a mixture of two normal distributions $\mathcal{N}(0, 1)$ and $\mathcal{N}(D, 1)$, where $D$ is the separation parameter. Intuitively, the larger the $D$, the smaller is $\xi_{\min}$, since it reduces connectivity of our measure. We also verify it numerically, as shown in Figure~\ref{fig:xi_gauss} (top). In this case, we utilize a simple two-layer MLP with spectral normalization on top of linear layers. We observe that indeed the value of $\xi_{\min}$ decays rapidly with the increase of the separation parameter $D$.

To experiment with GAN convergence in this toy setting, we setup two MLP models for $D$ and $G$, and train the LSGAN-like model with
\begin{equation*}
    \phi_1(x) = \frac{\alpha}{2}x^2+\beta x, \, \phi_2(x) = -\beta x, 
\end{equation*}
and gradient penalty with factor $\gamma$.
We train the model using the ODE-GAN approach with the Heun method. We experiment with two moderately separated mixtures, namely with $D=3$ and $D=4$. Respective $\xi_{\min}$ values obtained by the numerical simulation described above are $\xi_{\min}=0.25$ and $\xi_{\min}=0.124$. We consider five options for the values of $(\alpha, \beta, \gamma)$. We start with the baseline WGAN corresponding to $(0, -1, 0)$. We consider two options for the value of $\gamma$: the optimal one $\gamma=\sfrac{2|\beta|}{\sqrt{\xi_{\min}}}$ predicted by the theory, and a sub-optimal $\gamma=\sfrac{|\beta|}{\sqrt{\xi_{\min}}}$. With the latter value we obtain $\mathrm{Im} \ {\lambda_{\max}}\sim |\beta|\sqrt{\xi_{\min}}$. We also consider two LSGAN variants. First one is the default version with the parameters $(0.25, -0.5, 0)$. For the second version we fix $\beta=1$ and choose the optimal $\alpha$ and $\gamma$ according to \eqref{pdegan:optpar}. We measure performance of a GAN model via the Earth Mover's Distance between the Gaussians fitted on real and synthetic data. This approach resembles the commonly used Frech\'{e}t Inception Distance (FID) metric for GAN evaluation.
Results are visualized on Figure~\ref{fig:xi_gauss}. We observe that the resulting convergence plots match our theoretical predictions. In particular, we see that when the parameters are chosen in the optimal way, methods converge more rapidly and stably. On the other hand, when $\gamma$ is not tuned or is sub-optimal, the convergence is more oscillatory, and GANs struggle to converge. Note that even for the optimal $\gamma$ for methods with $\alpha > 0$, we still observe some oscillations. This may be a result of a noise induced by neural function approximation or stochasticity of training.

\begin{figure}[htb!]
    \centering
    \includegraphics[width=0.99\linewidth]{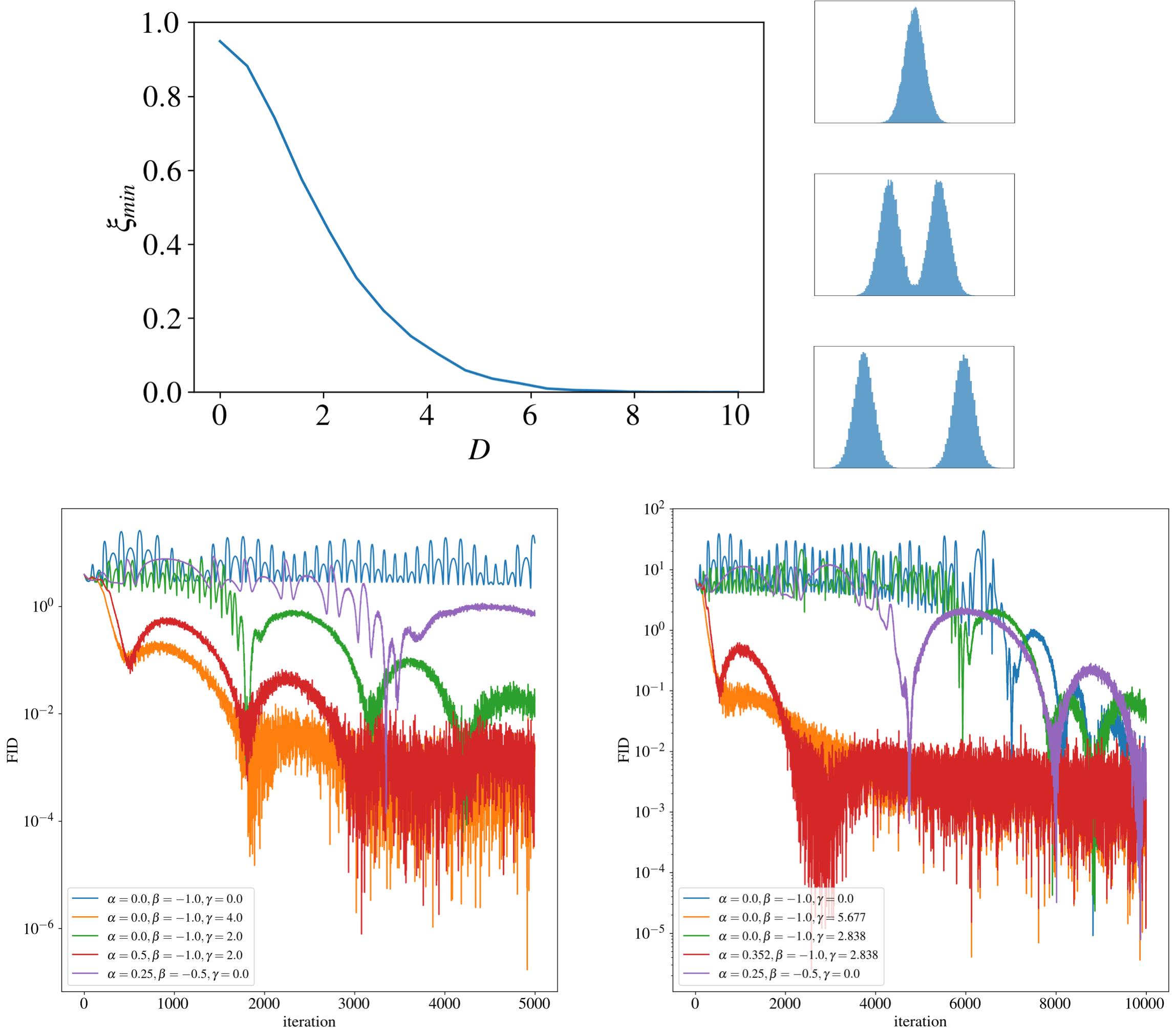}
    \caption{Numerical experiments with GAN convergence on a mixture of two univariate Gaussians. \textit{(Top)} Estimated value of $\xi_{\min}$ vs the mixture separating parameter $D$. \textit{(Bottom)} Convergence plots of GANs with different loss functions and regularization strengths. Methods with parameters selected optimally according to theory present better convergence.} 
    \label{fig:xi_gauss}
\end{figure}

\paragraph{Data augmentation of CIFAR-10.} In this set of experiments, we verify if the effects of data augmentations on GAN convergence can be predicted by evaluating $\xi_{\min}$ for the correspondingly augmented dataset. We consider a number of augmentations commonly utilized in data augmentation pipelines for GANs (see Figure~\ref{fig:aug}). 
\begin{figure}[htb!]
    \centering
    \includegraphics[width=0.99\linewidth]{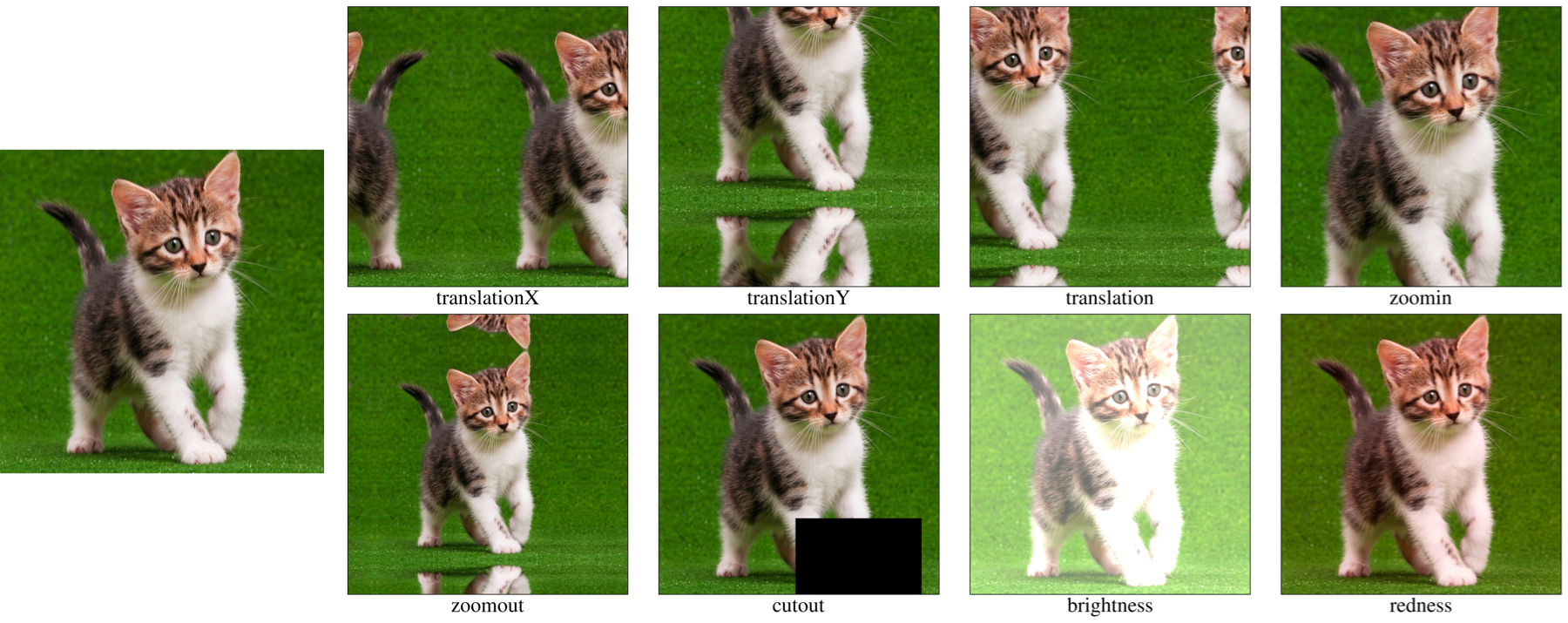}
    \caption{Various types of augmentations commonly used in GAN training.}
    \label{fig:aug}
\end{figure}
In particular, these augmentations include both spatial (translations, zooming) and visual augmentations (color adjustment). Folklore knowledge is that the first type is generally helpful, while the second is not. In our framework, the positive effect of augmentations can be understood as improving the ``connectivity'' of a dataset, which can be quantitatively measured by evaluating $\xi_{\min}$ of the augmented version of a dataset. Increased value of $\xi_{\min}$ both intuitively and theoretically (as supported by Theorem~\ref{pdegan:eigenfunthm}) corresponds to better convergence and better FID scores. Our experiments are based on a thorough analysis of the impact of augmentations on the quality of a GAN trained on CIFAR-10 provided in \citet{zhao2020image}. Specifically, the authors select a \emph{single} augmentation from a predefined set and train several types of GANs by augmenting real and fake images. The strength of the augmentation is controlled by a single parameter $\lambda$, and the authors provide the obtained FID for a number of its possible values (e.g., we may vary how strongly we zoom an image); we refer the reader to \citet{zhao2020image} for specific details on each augmentation. We selected a large portion of augmentations studied in this paper and replicated its data augmentation setup. For each augmentation and each strength value, we minimize the Rayleigh quotient with a neural network mimicking the SNDCGAN discriminator from \citet{zhao2020image}. We train it on the train part of the dataset by applying the respective augmentation to each image with probability one. Note that the actual augmentation strength is randomly sampled from the range $[0, \lambda]$, so we cover the entire distribution relatively well. For training, we use the batch size of $256$ and Adam optimizer with a learning rate $10^{-4}$ (results are not sensitive to these parameters). We measure the actual value of $\xi_{\min}$ by sweeping across the augmented test set $5$ times to better cover the distribution and aggregate the results across all $50000$ (augmented) samples. For the reference values, we choose the FID scores obtained by the SNDCGAN model with Balanced Consistency Regularization (bCR) compiled from \citet[Figure 3]{zhao2020image}. This model achieves better generation quality so our theory is more reliable in this case. Our results are provided at Figure~\ref{fig:cifar10_aug}. For convenience, we normalize the obtained values by dividing it by $\xi_{\min}$ of the non-augmented dataset (approximately equal to $0.0042$ by our estimation). We observe that FID scores in many cases indeed follow the behavior of $\xi_{\min}$. For instance, for \texttt{zoomin} and \texttt{cutout}, the estimated $\xi_{\min}$ values predict that there is an intermediate optimal augmentation strength, which is matched by the practice. For \texttt{translation}, we observe that stronger augmentations worsen the connectivity of the dataset, which results in worse FID scores. We can also observe a significant drop in $\xi_{\min}$ for color-based augmentations, confirming their practical inefficiency.
We note, in some cases (e.g., \texttt{zoomout}), we do not observe a direct correspondence. This may be a result of auxiliary effects not covered by our theory or of a discrepancy between ours and the authors' implementation.

\begin{figure}[htb!]
    \centering
    \includegraphics[width=0.99\linewidth]{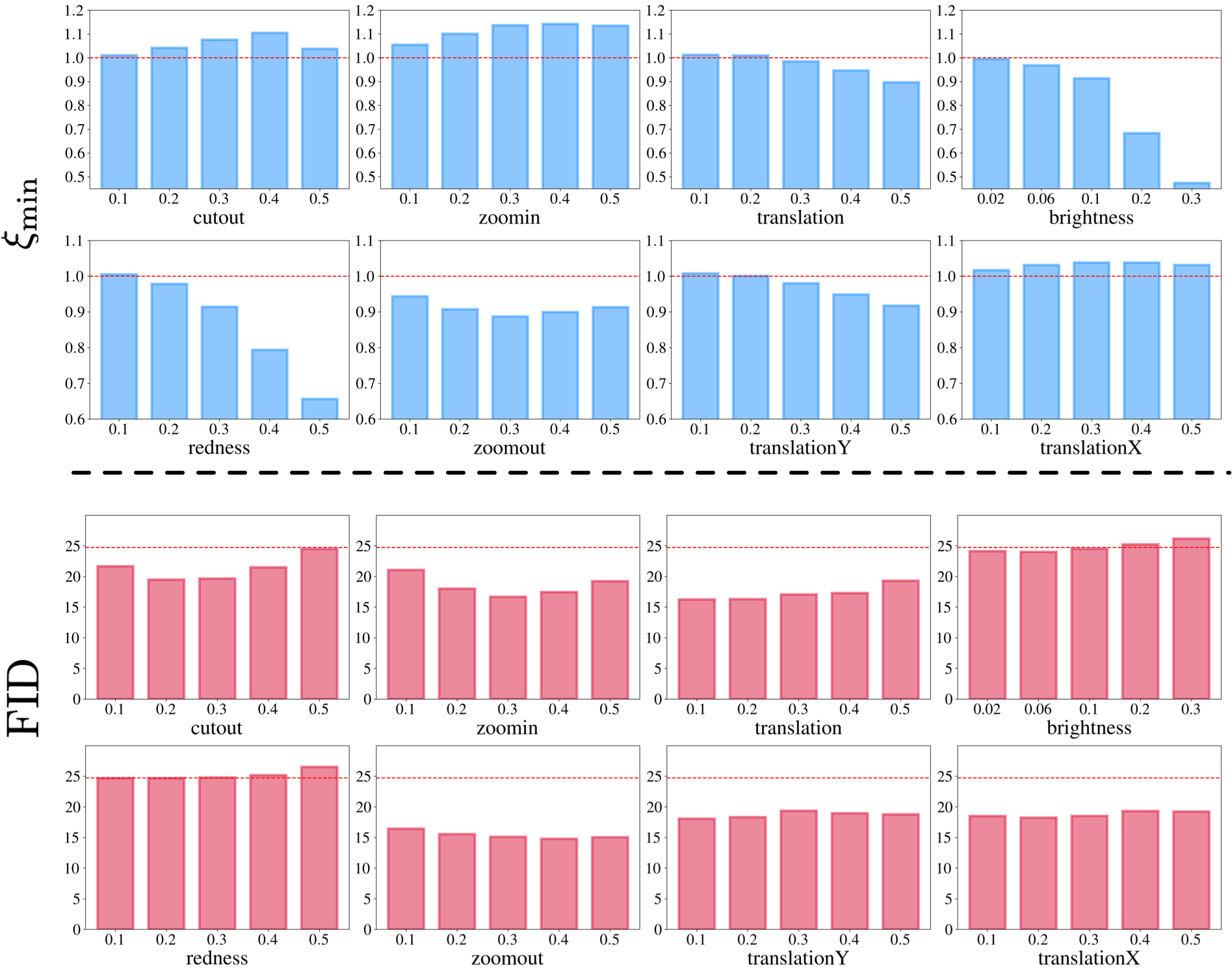}
    \caption{\textit{(Top)} The values of $\xi_{\min}$ for augmented versions of CIFAR-10 for different values of the augmentation strength. We normalize them dividing by $\xi_{\min}\approx 0.0042$ of the non-augmented version of the dataset. \textit{(Bottom)} The FID values of a SNDCGAN (with balanced consistency regularization) trained on the corresponding dataset. The values are taken from \citet{zhao2020image}. For $\xi_{\min}$ a higher value is ``better'' in terms of connectivity of dataset and theoretical convergence. For FID smaller values are better. In most cases we observe a negative rank correlation between $\xi_{\min}$ and FID.
    }
    \label{fig:cifar10_aug}
\end{figure}
\paragraph{Instance selection.}

Another possible way to manipulate the data distribution is to remove non-typical samples, which can confuse the generator \citep{devries2020instance}. This can be performed by fitting a Gaussian model on feature vectors of a dataset produced by some pretrained model (in our experiments, we used Inception v3 \citep{szegedy2016rethinking}) and keeping only samples with high likelihood under this model. This procedure results in better FID scores of trained GANs, as shown in \citet{devries2020instance}. In our framework, this can be understood as another way to improve the dataset's connectivity since by removing outliers, we reduce the number of gaps in the data. Experiments in the aforementioned paper were conducted on the resized ImageNet dataset; however, we hypothesize that this phenomenon may also be understood on the conceptually similar CIFAR-10 dataset. We perform an experiment confirming this effect quantitatively. We follow the steps described above, and for each value of $\psi \in \lbrace 0.1, 0.2, 0.3, 0.4, 0.5 \rbrace$ evaluate $\xi_{\min}$ of the dataset obtained by removing samples in the bottom $\psi$-quantile of the likelihood. Our results are provided in Table~\ref{tab:is}. We observe that for higher values of the truncation parameter, we indeed obtain better dataset connectivity, confirming practical findings of \citet{devries2020instance}.
\begin{table}
\centering
\caption{Normalized differences in $\xi_{\min}$ for CIFAR-10 with instance selection performed according to the procedure outlined in \citet{devries2020instance} with respect to the truncation (lower) quantile $\psi$. $\Delta \xi_{\min}$ is computed as $(\sfrac{\xi_{\min}(\psi)}{\xi_{\min}(0)} - 1) \times 10^2$ for easier comparison.
We observe that stronger truncation increases the value of $\xi_{\min}$, confirming better empirical GAN convergence. }
\label{tab:is}
\vspace{0.1cm}
\scalebox{1.0}{\begin{tabular}{lrrrrrr}
\toprule
$\psi$ & $0$ & $0.1$ & $0.2$ & $0.3$ & $0.4$ & $0.5$ \\
$\Delta \xi_{\min} \uparrow $ & $0$ & $0.95$ & $-0.04$ & $1.24$ & $\mathbf{2.90}$ & $2.74$ \\
\bottomrule

\end{tabular}}
\end{table}
\section{Conclusion}
We presented a novel framework for a theoretical understanding of GAN training by analyzing the local convergence in the functional space. Namely, we represent the GAN dynamics as a system of partial differential equations and analyze the spectrum of the corresponding differential operator, which determines the dynamics convergence. As the main result, we show how the spectrum depends on the properties of the target distribution, in particular, on its Poincar\'{e} constant. Our perspective provides a new understanding of established GAN tricks, such as gradient penalty or dataset augmentation. For practitioners, our paper develops an efficient method that allows to choose optimal augmentations for a particular dataset.

\bibliography{example_paper}
\bibliographystyle{icml2021}

\end{document}